\documentclass[journal]{IEEEtran}
\usepackage{amsthm}
\usepackage{url}
\usepackage{times}
\usepackage{algorithm}
\usepackage{algpseudocode}
\usepackage{verbatim}
\usepackage{graphicx}
\usepackage{xspace}
\newtheorem{theorem}{Theorem}
\newtheorem{lemma}{Lemma}

\usepackage[numbers]{natbib}
\usepackage{drum}

\providecommand{\algorithmtopspace}

\newcommand{\nbsp}{{\color{white}.}}
\newcommand{\RPISIM}{\texttt{RPIsim}\xspace}
\newcommand{\Moby}{\texttt{Moby}\xspace}
\newcommand{\ODE}{\texttt{ODE}\xspace}
\newcommand{\Bullet}{\texttt{Bullet}\xspace}
\newcommand{\LEMKE}{\texttt{LEMKE}\xspace}
\newcommand{\urlx}[1]{\texttt{\url{#1}}}

\begin{document}
%
\title{Rapidly computable viscous friction and no-slip\\ rigid contact models}
%
%
%

\author{Evan Drumwright
\thanks{E. Drumwright is with the Department
of Computer Science, George Washington University, Washington,
DC \{drum@gwu.edu\}}
\thanks{Manuscript received April 19, 2005; revised December 27, 2012.}}

\markboth{Journal of \LaTeX\ Class Files,~Vol.~11, No.~4, December~2012}%
{Shell \MakeLowercase{\textit{et al.}}: Bare Demo of IEEEtran.cls for Journals}
%

\maketitle

\begin{abstract}
This article presents computationally efficient algorithms for modeling two special cases of rigid contact---contact with only viscous friction and contact without slip---that have particularly useful applications in robotic locomotion and grasping. Modeling rigid contact with Coulomb friction generally exhibits $O(n^3)$ expected time complexity in the number of contact points and $2^{O(n)}$ worst-case complexity. The special cases we consider exhibit $O(m^3 + m^2n)$ time complexity ($m$ is the number of independent coordinates in the multi rigid body system) in the expected case and polynomial complexity in the worst case; thus, asymptotic complexity is no longer driven by number of contact points (which is conceivably limitless) but instead is more dependent on the number of bodies in the system (which is often fixed). These special cases also require considerably fewer constrained nonlinear optimization variables thus yielding substantial improvements in running time. Finally, these special cases also afford one other advantage: the nonlinear optimization problems are numerically easier to solve.  
\end{abstract}

\IEEEpeerreviewmaketitle

\section{Introduction}
Dynamic robotic simulation; grasp planning; and, increasingly, locomotion planning and control employ rigid contact models with dry (typically Coulomb) and wet (viscous) friction. These contact models yield an effective tradeoff between computation speed and physical accuracy. While rigid contact models are far faster than, \emph{e.g.}, elastodynamic finite element analysis, they still require heavy computation: the expected time complexity for such models is $O(n^3)$ in the number of contact points. Additionally, the number of contact points input to the model is conceivably limitless. This issue is not just theoretically interesting: Wang~\cite{Wang:2013} reports that solving the contact problem absorbs up to 90\% of computation time when simulating a scenario for the DARPA Robotics Challenge using \ODE~\cite{Smith:nw}.

Roboticists are often content to use rigid contact models without
Coulomb friction for computational expediency. For example, one may wish to model locomotion or effect simulated grasping without observing slip; roboticists studying legged locomotion often predicate their models on no slip occurring, for example. If slip is desirable, purely viscous friction might yield a suitable model if, for example, a robot is walking on a wet surface. This article presents computationally efficient methods for both of these special cases.

These special cases provide the following computational and modeling advantages: \1 time complexity goes from worst-case exponential (the worst-case complexity of solving rigid contact problems with Coulomb friction~\cite{Stewart:1996,Anitescu:1997} using Lemke's Algorithm~\cite{Murty:1988}) to worst-case polynomial in the number of contact points; \2 significant reduction in the number of nonlinear optimization problem variables; and \3 a positive-semi-definite-matrix linear complementarity problem (LCP), in place of a copositive-plus LCP, which is demonstrably easier to solve~\cite{Drumwright:2012} (\emph{i.e.}, the solver is less likely to fail due to numerical errors) and permits the use of general algorithms for solving convex optimization problems.

Finally, we provide an algorithm that yields $O(m^3 + m^2n)$ expected asymptotic time complexity on these two contact models, where $m$ is the number of independent coordinates in the multi rigid body system. This algorithm therefore provides a means to make complexity more dependent on the number of independent coordinates in the system (this number remains constant except in the unusual case in which bodies are inserted into the simulation) than on the number of contact points (which is conceivably unlimited).

\section{LCPs, NCPs, and MLCPs}
\label{section:LCPs}
A LCP, or linear complementarity problem, ($\vect{r}, \mat{Q}$) signifies the problem:
\begin{align}
\vect{w} & = \mat{Q}\vect{z} + \vect{r} \nonumber \\
\vect{w} & \geq \vect{0} \nonumber \\
\vect{z} & \geq \vect{0} \nonumber \\
\tr{\vect{z}}\vect{w} & = 0 \nonumber
\end{align}
for unknown vectors $\vect{z}, \vect{w} \in \mathbb{R}^q$. 

A nonlinear complementarity problem (NCP) is composed of a number of nonlinear complementarity constraints~\cite{Cottle:1992} that take the form:
\begin{align}
\vect{x} & \geq \vect{0}\\
f(\vect{x}) & \geq \vect{0}\\
\tr{\vect{x}}f(\vect{x}) & = 0
\end{align}
where $\vect{x} \in \mathbb{R}^q$ and $f : \mathbb{R}^q \to \mathbb{R}^q$. 

A mixed linear complementarity problem (MLCP) is defined by the following constraints:
\begin{align}
\mat{A}\vect{x} + \mat{C}\vect{y} + \vect{g} & = \vect{0} \label{eqn:MLCP-begin} \\
\mat{D}\vect{x} + \mat{B}\vect{y} + \vect{h} & \geq \vect{0} \\
\vect{y} & \geq \vect{0} \\
\tr{\vect{y}}(\mat{C}\vect{x} + \mat{D}\vect{y} + \vect{h}) & = 0 \label{eqn:MLCP-end}
\end{align}
Note that the $\vect{x}$ variables are unconstrained, while the $\vect{y}$ variables must be non-negative. If $\mat{A}$ is non-singular, the unconstrained variables can be computed as:
\begin{align}
\vect{x} = -\inv{\mat{A}}(\mat{C}\vect{y} + \vect{g}) \label{eqn:MLCP-free}
\end{align}
Substituting $\vect{x}$ into equations \ref{eqn:MLCP-begin}--\ref{eqn:MLCP-end} yields the LCP ($\vect{e}, \mat{F}$):
\begin{align}
\mat{F} & \equiv \mat{B} - \mat{D}\inv{\mat{A}}\mat{C} \label{eqn:MLCP-LCP1} \\
\vect{e} & \equiv \vect{h} - \mat{D}\inv{\mat{A}}\vect{g} \label{eqn:MLCP-LCP2}
\end{align}
A solution $(\vect{y}, \vect{\nu})$ to this LCP obeys the relationship $\mat{F}\vect{y} + \vect{e} = \vect{\nu}$; once one has $\vect{y}$, $\vect{x}$ may be determined via Equation~\ref{eqn:MLCP-free}, and the MCLP is solved.

\section{Background}

\subsection{Coulomb friction}
Coulomb's friction model provides relationships between the force applied along the contact normal and the frictional forces. Coulomb friction considers two cases, \emph{rolling/sticking} and \emph{sliding}. The former occurs when the velocity is zero in the tangent plane of the contact frame; conversely, sliding occurs when that velocity is non-zero.

The magnitude of the friction force for a sliding contact modeled with Coulomb friction is given by the equation:
\begin{equation}
f_f = \mu_c f_n
\end{equation}
where $f_n$ is the magnitude of the force applied along the contact normal. The frictional force is applied directly opposite the direction of sliding (\emph{i.e.}, against the relative velocity in the tangent plane of the contact frame).

The magnitude of the friction force for a rolling or sticking contact modeled with Coulomb friction is given by the equation:
\begin{equation}
f_f \leq \mu_c f_n
\end{equation}
In the case of rolling/sticking friction, the friction force acts to resist motion in the tangent (\emph{e.g.}, in the case of a box resting on a slope); thus, $f_f$ may be strictly less than $\mu_c f_n$. If external forces become sufficiently large to overcome rolling/sticking friction forces, the rolling/sticking contact will transition to sliding.

Many applications in robotics use the Coulomb friction model because it is
relatively straightforward to compute---one can determine the frictional forces without integrating ordinary differential equations---and it is reasonably predictive.
Nevertheless, Coulomb friction is somewhat expensive (computationally) to model: the rigid contact models of Stewart and Trinkle~\cite{Stewart:1996} and Anitescu and Potra~\cite{Anitescu:1997} can be solved in expected polynomial time in the $n$ contacts\footnote{These models yield an order $n$ copositive-plus LCP solvable by Lemke's Algorithm~\cite{Lemke:1965}. Each iteration of Lemke's Algorithm requires an $O(n^2)$ matrix factorization update, and $n$ iterations of the algorithm are expected~\cite{Cottle:1992}.}, though exponential complexity may be exhibited in the worst case.

\subsection{Acceleration-level rigid body contact model with Coulomb and viscous friction}
\label{section:accel-level}
We now describe the rigid contact model with Coulomb and viscous friction that uses only non-impulsive forces for consistency with the \emph{principle of constraints}~\cite{Kilmister:1966}. The multi rigid body dynamics equation with contact and joint constraint forces is given below: 
\begin{align}
\mat{M}(t)\dot{\vect{v}} = & \vect{f}(t) + \tr{\mat{J}(t)}\vect{f}_j + \tr{\mat{N}(t)}\vect{f}_n + \ldots \label{eqn:accel-system-dynamics} \\
& \tr{\mat{S}_k(t)}\vect{f}_s + \tr{\mat{T}_k(t)}\vect{f}_t - \tr{\mat{Q}(t)}\vect{f}_q - \ldots \label{eqn:accel-Coulomb} \\
& \tr{\mat{S}(t)}\mat{\mu}_v \mat{S}(t)\vect{v} -  \tr{\mat{T}(t)}\mat{\mu}_v \mat{T}(t)\vect{v} \label{eqn:accel-viscous}
\end{align}
where $\mat{M} \in \mathbb{R}^{m \times m}$ is the system inertia matrix; $\vect{v} \in \mathbb{R}^{m}$ is the system velocity; $\mat{J} \in \mathbb{R}^{j \times m}$ is the matrix of $j$ bilateral constraint equations; $\mat{N} \in \mathbb{R}^{n \times m}$, $\mat{S} \in \mathbb{R}^{n \times m}$, and $\mat{T} \in \mathbb{R}^{n \times m}$ are matrices of $n$ wrenches applied along the contact normal, first contact tangent, and second contact tangent, respectively; $\mat{Q} \in \mathbb{R}^{r \times m}$ is a matrix of $r$ generalized wrenches applied against the direction of sliding for $r \leq n$ sliding contacts; $\vect{f}_j \in \mathbb{R}^j$ is the vector bilateral constraint force magnitudes; $\vect{f}_n \in \mathbb{R}^n$ is a vector of contact normal force magnitudes; $\vect{f}_s \in \mathbb{R}^k$ and $\vect{f}_t \in \mathbb{R}^k$ are vectors of contact tangent force magnitudes applied at the $k = n - r$ rolling/sticking contacts; $\vect{f}_q \in \mathbb{R}^r$ is a vector of contact tangent force magnitudes applied at the $r$ sliding contacts; $\vect{f}$ is a vector of forces on the rigid body system (gravity, Coriolis and centrifugal forces, \emph{etc.}); and $\mat{\mu}_v$ is a diagonal matrix of viscous friction coefficients.

Equation~\ref{eqn:accel-Coulomb} specifies the Coulomb friction forces and Equation~\ref{eqn:accel-viscous} specifies the viscous friction forces. The viscous model, where friction forces oppose the direction of motion, is commonly used in robotics (see, \emph{e.g.},~\cite{Sciavicco:2000am}) and is a simplification of the viscous drag term in fluid dynamics. 
Out of the $n$ points of contact in the system, some may be rolling/sticking and the remainder will be sliding. For Coulomb friction, the first two terms of Equation~\ref{eqn:accel-Coulomb} ($\tr{\mat{S}_k(t)}\vect{f}_s + \tr{\mat{T}_k(t)}\vect{f}_t$) are relevant to the rolling/sticking contacts only ($k$ specifies the indices of $\mat{S}$ and $\mat{T}$ that correspond to rolling/sticking contacts) and the last term ($-\tr{\mat{Q}(t)}\vect{f}_q$) is relevant to only sliding contacts.

\subsubsection{Bilateral constraint equation}
\label{section:accel-bilat}
Bilateral constraints can be specified in the form $\phi(\vect{q}) = \vect{0}$, where $\vect{q}$ are the generalized coordinates of the system (joint constraints that are an explicit function of time are not considered here, though their inclusion would not change the results in this article). Such constraints can be differentiated once with respect to time to yield:
\begin{equation}
\mat{J}\dot{\vect{v}} = \vect{0}
\end{equation}
where $\mat{J} \equiv \frac{\partial \phi}{\partial \vect{q}}$. If we differentiate the constraints with respect to time once more, the bilateral joint constraints can be enforced using the equation:
\begin{equation}
\mat{J}\dot{\vect{v}} + \dot{\mat{J}}\vect{v} = \vect{0} \label{eqn:bilat}
\end{equation}
where $\dot{\mat{J}} \equiv \frac{\partial}{\partial \vect{q}}{\mat{J}\vect{v}}$.

\subsubsection{Contact normal constraints}
\label{section:accel-contact-normal}
We assume that there are $n$ points of contact. The $i^{\textrm{th}}$ contact must satisfy the following linear complementarity condition that relates normal force and non-interpenetration:
\begin{equation}
0 \leq f_{n_i} \perp \tr{\vect{n}_i}\dot{\vect{v}} + \tr{\dot{\vect{n}_i}}\vect{v} \geq 0
\label{eqn:normal-comp-accel}
\end{equation}
where $\vect{n}_i$ and $\dot{\vect{n}}_i$ are column vectors taken from the $i^{\textrm{th}}$ rows of $\mat{N}$ and $\dot{\mat{N}}$, respectively. 
Here we adopt the notation $a \perp b$ to denote the relationship $a\cdot b = 0$.
$f_{n_i} \geq 0$ requires that the contact force can only push bodies apart, \mbox{$\tr{\vect{n}_i}\dot{\vect{v}} + \dot{\tr{\vect{n}_i}}\vect{v} \geq 0$} requires that the bodies cannot be accelerating toward one another at the $i^{th}$ contact point after the contact forces are applied, and the $f_{n_i} \cdot \  (\tr{\vect{n}_i}\dot{\vect{v}} + \tr{\dot{\vect{n}_i}}\vect{v}) = 0$ constraint ensures that frictionless contact does no work. 

\subsubsection{Sliding friction}
If the velocity in the tangent plane at the $i^{\textrm{th}}$ contact point is non-zero, then the contact is sliding, and the Coulomb friction model specifies the magnitude of force to be applied. 
\begin{equation}
|f_{q_i}| = \mu |f_{n_i}|
\end{equation} 

\subsubsection{Rolling/sticking friction}
If the velocity at time $t$ in the tangent plane at the $i^{\textrm{th}}$ contact point is zero, then the contact is rolling/sticking at time $t$ and may either continue rolling/sticking or begin sliding, depending on the normal force and Coulomb friction coefficient. The nonlinear complementarity conditions are then expressed by the following equations (adapted from \cite{Trinkle:1997}):
\begin{align}
0 & \leq u^2f_{n_i}^2 - f_{s_i}^2 - f_{t_i}^2\ \bot\ \sqrt{\dot{v}_{s_i}^2 + \dot{v}_{t_i}^2} \geq 0 \label{eqn:coulomb-accel} \\
0 & = \mu f_{n_i} \dot{v}_{s_i} + f_{s_i}\sqrt{\dot{v}_{s_i}^2 + \dot{v}_{t_i}^2} \label{eqn:fs-accel} \\
0 & = \mu f_{n_i} \dot{v}_{t_i} + f_{t_i}\sqrt{\dot{v}_{s_i}^2 + \dot{v}_{t_i}^2} \label{eqn:ft-accel}
\end{align}
Let us now examine the constraints above.  Equation~\ref{eqn:coulomb-accel} constrains the frictional force to lie within the friction cone; if the tangential acceleration is non-zero, then the frictional force must lie on the edge of the friction cone. Equations~\ref{eqn:fs-accel}~and~\ref{eqn:ft-accel} ensure that the frictional force opposes the tangent acceleration.

\subsubsection{Solvability of the model}
Others (\emph{e.g.}, \cite{Baraff:1994}) have already shown that this rigid model may not possess a solution if there are any sliding contacts; such contact scenarios are known as \emph{inconsistent configurations}. Nevertheless, we present this model because the contact model with Coulomb friction, which we present next and use to motivate the move to a velocity-level contact model, will build off of it.

\subsection{Solvable rigid body contact model with Coulomb and viscous friction}
The contact model of Stewart and Trinkle~\cite{Stewart:1996} and Anitescu and Potra~\cite{Anitescu:1997} provides a guaranteed solution to the problem of inconsistent configurations in contact models with Coulomb friction. This model is presented to show a  velocity-level formulation, which allows the model to overcome the issue of inconsistent configurations. The no-slip model introduced in Section~\ref{section:no-slip} will also employ a velocity-level formulation to simulate contact without sliding in arbitrary configurations; Lynch and Mason showed that sliding with infinite friction is possible for the acceleration-level model described in the previous section~\cite{Lynch:1995}.

We now describe this contact model---we consider only the aspect of the model that treats all contacts as inelastic impacts and do not consider extensions to collisional impacts with restitution. For simplicity of presentation, we do not linearize the friction cone, which yields a NCP rather than the LCP in~\cite{Stewart:1996,Anitescu:1997}. 

The contact model uses a first-order approximation to the rigid body dynamics to resolve issues like Painlev\'{e}'s Paradox (and other inconsistent contact configurations~\cite{Baraff:1994}), for which no non-impulsive force solutions exist. The rigid body dynamics are given by:
\begin{align}
\mat{M}(t)\Delta \vect{v} = & \Delta t \vect{f}(t) + \tr{\mat{J}(t)}\vect{f}_j + \ldots \label{eqn:vel-system-dynamics} \\
& \tr{\mat{N}(t)}\vect{f}_n + \tr{\mat{S}(t)}\vect{f}_s + \tr{\mat{T}(t)}\vect{f}_t \nonumber
\end{align}
where $\mat{M}$, $\vect{v}$, $\mat{J}$, $\mat{N}$, $\mat{S}$, $\mat{T}$, $\vect{f}_j$, $\vect{f}_n$, $\vect{f}_s$, $\vect{f}_t$, and $\vect{f}$ are as defined in Section~\ref{section:accel-level} and $\Delta t$ is the change in time that realizes the first-order approximation. We now define $\vect{v}^* \equiv \vect{v} + \Delta \vect{v}$.

\subsubsection{Bilateral constraint equation}
Because the bilateral joint constraints are now defined at the velocity level, the constraints are enforced using the equation:
\begin{equation}
\mat{J}\vect{v}^* = \vect{0} \label{eqn:vel-bilateral-constraint}
\end{equation}

\subsubsection{Contact normal constraints}
The velocity-level constraints on contact normal force and non-interpenetration are now defined as:
\begin{equation}
\label{eqn:vel-normal-comp}
0 \leq f_{n_i} \perp \tr{\vect{n}}_i\vect{v}^* \geq 0 
\end{equation}

\subsubsection{Coulomb friction constraints}
Coulomb friction is effected more simply in this model than in the acceleration-level model: contacts can be treated identically whether they are initially sliding or sticking. The nonlinear complementarity conditions for the $i^{\textrm{th}}$ contact are:
\begin{align}
0 & \leq u^2f_{n_i}^2 - f_{s_i}^2 - f_{t_i}^2\ \bot\ \sqrt{v_{s_i}^{*^2} + v_{t_i}^{^*2}} \geq 0 \\
0 & = \mu f_{n_i} v_{s_i}^* + f_{s_i}\sqrt{v_{s_i}^{*^2} + v_{t_i}^{*^2}} \\
0 & = \mu f_{n_i} v_{t_i}^* + f_{t_i}\sqrt{v_{s_i}^{*^2} + v_{t_i}^{*^2}}
\end{align}
These equations are analogous to Equations~\ref{eqn:coulomb-accel}--\ref{eqn:ft-accel}.

If the nonlinear complementarity conditions are converted to linear complementarity constraints by use of a linearized friction cone (\emph{i.e.}, a friction polygon), then a provably solvable copositive-plus LCP~\cite{Cottle:1992} results. However, Lemke's Algorithm~\cite{Lemke:1965} is currently the only algorithm provably capable of solving copositive-plus LCPs. Lemke's Algorithm can exhibit exponential complexity~\cite{Murty:1988}, though polynomial time is expected.

\section{Contact model with purely viscous friction}
\label{section:viscous-model}
We now describe the contact model with purely viscous friction. We start from the multi rigid body dynamics equation at the acceleration level (Equations~\ref{eqn:accel-system-dynamics}~and~\ref{eqn:accel-viscous}), which are reproduced below:
\begin{align}
\mat{M}(t)\dot{\vect{v}} = & \vect{f}(t) + \tr{\mat{J}(t)}\vect{f}_j + \tr{\mat{N}(t)}\vect{f}_n + \ldots \nonumber \\
& \tr{\mat{S}(t)}\mat{\mu}_v \mat{S}(t)\vect{v} -  \tr{\mat{T}(t)}\mat{\mu}_v \mat{T}(t)\vect{v} \nonumber
\end{align}

To this we add bilateral constraints (Equation~\ref{section:accel-bilat}) and the normal contact compressive force and non-interpenetration and complementarity constraints (Equation~\ref{section:accel-contact-normal}), again reproduced below:
\begin{align*}
\mat{J}\dot{\vect{v}} + \dot{\mat{J}}\vect{v} & = \vect{0} \\
0 \leq \tr{\vect{n}_i}\dot{\vect{v}} + \tr{\dot{\vect{n}_i}}\vect{v} & \perp f_{n_i} \geq 0 \ \textrm{ for } i = 1,\ldots, n
\end{align*}

Combining these equations yields the following MLCP:
\begin{align}
\begin{bmatrix}\mat{M} & -\tr{\mat{J}} & -\tr{\mat{N}} \\
\mat{J} & \vect{0} & \vect{0} \\
\mat{N} & \vect{0} & \mat{0}
\end{bmatrix}
\begin{bmatrix}
\dot{\vect{v}} \\
\vect{f}_j \\
\vect{f}_n
\end{bmatrix}
+
\begin{bmatrix}
\vect{f}^* \\
\dot{\mat{J}}\vect{v} \\
\dot{\mat{N}}\vect{v}
\end{bmatrix}
& = 
\begin{bmatrix}
\vect{0} \\
\vect{0} \\
\vect{\gamma} 
\end{bmatrix}\\
\vect{f}_n & \geq \vect{0} \\
\vect{\gamma} & \geq \vect{0} \\
\tr{\vect{f}}_n\vect{\gamma} & = 0
\end{align}

where $\vect{f}^* \equiv -\vect{f} + \tr{\mat{S}}\mat{\mu}_v\mat{S}\vect{v} + \tr{\mat{T}}\mat{\mu}_v\mat{T}\vect{v}$ and $\vect{\gamma} \equiv \mat{N}\dot{\vect{v}} + \dot{\mat{N}}\vect{v}$. As long as $\mat{J}$ has full row rank (we will describe how to ensure this condition in the next section), the mixed LCP can be converted to a conventional LCP (as described in Section~\ref{section:LCPs}) using the following definitions:
\begin{align}
\mat{A} & \equiv \begin{bmatrix}
\mat{M} & -\tr{\mat{J}} \\
\mat{J}  & \mat{0}
\end{bmatrix}\\
\mat{C} & \equiv \begin{bmatrix}
-\tr{\mat{N}} \\
\mat{0}
\end{bmatrix}\\
\mat{D} & \equiv -\tr{\mat{C}} \\
\mat{B} & \equiv \mat{0}\\
\vect{x} & \equiv \begin{bmatrix} \dot{\vect{v}} \\ \vect{f}_j  \end{bmatrix}\\
\vect{y} & \equiv \vect{f}_n\\
\vect{g} & \equiv \begin{bmatrix} -\vect{f}^* \\ \vect{0} \end{bmatrix}\\
\vect{h} & \equiv \vect{0}
\end{align} 
Equations~\ref{eqn:MLCP-LCP1}~and~\ref{eqn:MLCP-LCP2} then yield the following standard LCP:
\begin{align}
\mat{F} & \equiv \mat{N}\inv{\mat{A}}\tr{\mat{N}} \\
\vect{e} & \equiv \mat{N}\inv{\mat{A}}\vect{f}^*
\end{align}
The system inertia matrix is block diagonal (each block is invertible), so $\mat{A}$ is invertible if $\mat{J}$ has full row rank (if it is not---indicating that one or more constraints is redundant---a subset of $\mat{J}$ which has full row rank can be used to ensure that $\mat{A}$ is invertible). From~\cite{Bhatia:2007}, a matrix of $\mat{F}$'s form must be non-negative definite, \emph{i.e.}, either positive semi-definite (PSD) or positive definite (PD). Additionally, Baraff provided an algorithm that provably solved LCPs of the form $(\mat{G}\vect{r}, \mat{G}\mat{H}\tr{\mat{G}})$, where $\mat{H} \in \mathbb{R}^{m \times m}$ is a symmetric matrix, $\vect{r} \in \mathbb{R}^m$, and $\mat{G} \in \mathbb{R}^{n \times m}$~\cite{Baraff:1994}. Finally, we note that LCPs with PSD/PD matrices are equivalent to convex quadratic programs~\cite{Cottle:1992}, which means that solving the LCP exhibits worst-case polynomial computational complexity.

\section{Reducing expected time complexity \\from $O(n^3)$ to $O(m^3 + m^2n)$}
A system with $m$ degrees-of-freedom requires no more than $m$ positive force magnitudes applied along the contact normals to satisfy the constraints for the contact models with purely viscous friction and without slip (the latter model will be presented in Section~\ref{section:no-slip}). We now prove this statement.

Assume we permute and partition the rows of $\mat{N}$ into $r$ linearly independent and $n-r$ linearly dependent rows, denoted by indices $I$ and $D$, respectively, as follows:
\begin{equation}
\mat{N} = \begin{bmatrix}\mat{N}_I \\ \mat{N}_D \end{bmatrix}
\end{equation}
Then the LCP vectors $\vect{q} = \mat{N}\vect{v}$, $\vect{z} \in \mathbb{R}^n$, and $\vect{w} \in \mathbb{R}^n$ and LCP matrix $\mat{Q}=\mat{N}\inv{\mat{A}}\tr{\mat{N}}$ can be partitioned as follows:
\begin{equation}
\begin{bmatrix} \mat{Q}_{II} & \mat{Q}_{ID} \\ \mat{Q}_{DI} & \mat{Q}_{DD} \end{bmatrix}\begin{bmatrix}\vect{z}_I \\ \vect{z}_D \end{bmatrix} + \begin{bmatrix}\vect{q}_I \\ \vect{q}_D \end{bmatrix} = \begin{bmatrix}\vect{w}_I \\ \vect{w}_D \end{bmatrix}
\end{equation}
Given some matrix $\alpha \in \mathbb{R}^{(n-r) \times r}$, it is the case that $\mat{N}_{D} = \mat{\alpha}\mat{N}_{I}$, and therefore that $\mat{Q}_{DI} =  \mat{\alpha}\mat{N}_I\inv{\mat{A}}\tr{\mat{N}_{I}}$, $\mat{Q}_{ID} =  \mat{N}_I\inv{\mat{A}}\tr{\mat{N}_{I}}\tr{\mat{\alpha}}$ (by symmetry), $\mat{Q}_{DD} = \mat{\alpha}\mat{N}_I\inv{\mat{A}}\tr{\mat{N}_{I}}\tr{\mat{\alpha}}$, and $\vect{q}_D = \mat{\alpha}\mat{N}_I\vect{v}$.

\begin{lemma}
\label{lemma:rank}
Since $\rank{\mat{X}\mat{Y}} \leq \min\left(\rank{\mat{X}},\rank{\mat{Y}}\right)$, the number of positive components of $\vect{z}_I$ can not be greater than rank$(\mat{A})$.
\end{lemma}
\begin{proof}
Since the columns of $\mat{X}\mat{Y}$ have $\mat{X}$ multiplied by each column of $\mat{Y}$, \emph{i.e.}, $\mat{X}\mat{Y} = \begin{bmatrix} \mat{X}\vect{y}_1 & \mat{X}\vect{y}_2 & \ldots & \mat{X}\vect{y}_n \end{bmatrix}$. Columns in $\mat{Y}$ that are linearly dependent will thus produce columns in $\mat{X}\mat{Y}$ that are linearly dependent (with precisely the same coefficients). Thus, rank($\mat{X}\mat{Y}$) $\leq$ rank($\mat{Y}$). Applying the same argument to the transposes produces\\ rank($\mat{X}\mat{Y}$) $\leq$ rank($\mat{X}$), thereby proving the claim.
\end{proof}
We now show that---in the case that the number of positive components of $\vect{z}_I$ is equal to the rank of $\mat{A}$---no more positive force magnitudes are necessary to solve the LCP.

\begin{theorem}
\label{thm:maxcard}
If $(\vect{z}_I = \vect{a}, \vect{w}_I = \vect{0})$ is a solution to the LCP $(\vect{q}_I, \mat{Q}_{II})$, then $(\begin{bmatrix}\tr{\vect{z}_I} = \tr{\vect{a}} & \tr{\vect{z}_D} = \tr{\vect{0}}\end{bmatrix}^{\mathsf{T}}, \vect{w} = \vect{0})$ is a solution to the LCP $(\vect{q}, \mat{Q})$.
\end{theorem}
\begin{proof}
For $(\begin{bmatrix}\tr{\vect{z}_I} = \tr{\vect{a}} & \tr{\vect{z}_D} = \tr{\vect{0}}\end{bmatrix}^{\mathsf{T}}, \vect{w} = \vect{0})$ to be a solution to the LCP $(\vect{q}, \mat{Q})$, six conditions must be satisfied:
\begin{enumerate}
\item $\vect{z}_I \geq \vect{0}$
\item $\vect{w}_I \geq \vect{0}$
\item $\tr{\vect{z}_I}\vect{w}_I = 0$
\item $\vect{z}_D \geq \vect{0}$
\item $\vect{w}_D \geq \vect{0}$
\item $\tr{\vect{z}_D}\vect{w}_D = 0$
\end{enumerate}
Of these, \1\,, \4, and \6 are met trivially by the assumptions of the theorem. Since $\vect{z}_D = \vect{0}$, $\mat{Q}_{II}\vect{z}_I + \mat{Q}_{ID}\vect{z}_D + \vect{q}_I = \vect{0}$, and thus $\vect{w}_I = \vect{0}$, thus satisfying \2 and \3. Also due to $\vect{z}_D = \vect{0}$, it suffices to show for \5 that $\mat{Q}_{DI}\vect{z}_I + \vect{q}_D \geq \vect{0}$. From above, the left hand side of this equation is equivalent to $\mat{\alpha}(\mat{N}_I\inv{\mat{A}}\tr{\mat{N}_I}\vect{a} + \mat{N}_I\vect{v})$, or $\mat{\alpha} \vect{w}_I$, which itself is equivalent to $\mat{\alpha} \vect{0}$. Thus, $\vect{w}_D = \vect{0}$.
\end{proof}

\subsubsection{Algorithm}
We use the theorem above to make a minor modification to the Principal Pivot Method I~\cite{Cottle:1968,Murty:1988} (PPM), which solves LCPs with $P$-matrices (complex square matrices with fully non-negative principal minors~\cite{Murty:1988} that includes positive semi-definite matrices as a proper subset). The resulting algorithm limits the size of matrix solves and multiplications.

The PPM uses a set $\beta$ with maximum cardinality $n$ for a LCP of order $n$. Of a pair of LCP variables, $(z_i, w_i)$, exactly one will be in $\beta$; we say that the other belongs to $\overline{\beta}$. If a variable belongs to $\beta$, we say that the variable is a \emph{basic variable}; otherwise, it is a \emph{non-basic variable}. Using this set, partition the LCP matrices and vectors as shown below:
\begin{equation}
\begin{bmatrix}
\vect{w}_{\beta} \\
\vect{w}_{\overline{\beta}}
\end{bmatrix} =
\begin{bmatrix}
\mat{A}_{\beta \beta} & \mat{A}_{\beta \overline{\beta}} \\ 
\mat{A}_{\overline{\beta} \beta} & \mat{A}_{\overline{\beta} \overline{\beta}} 
\end{bmatrix}
\begin{bmatrix}
\vect{z}_{\beta}\\
\vect{z}_{\overline{\beta}} 
\end{bmatrix}
+
\begin{bmatrix}
\vect{q}_{\beta}\\
\vect{q}_{\overline{\beta}}
\end{bmatrix} \nonumber
\end{equation}
Segregating the basic and non-basic variables on different sides yields:
\begin{align}
\begin{bmatrix}
\vect{w}_{\beta} \\
\vect{z}_{\beta}
\end{bmatrix} = &
\begin{bmatrix}
\mat{A}_{\beta \overline{\beta}} - \mat{A}_{\beta \beta}\inv{\mat{A}_{\overline{\beta} \beta}} & \mat{A}_{\beta \beta}\inv{\mat{A}_{\overline{\beta} \beta}} \\ 
-\inv{\mat{A}_{\overline{\beta} \beta}}\mat{A}_{\overline{\beta} \overline{\beta}} & \inv{\mat{A}_{\overline{\beta} \beta}}
\end{bmatrix}
\begin{bmatrix}
\vect{z}_{\overline{\beta}} \\
\vect{w}_{\overline{\beta}} 
\end{bmatrix}
+ \ldots \nonumber \\
& \quad
\begin{bmatrix}
\vect{q}_{\beta} - \mat{A}_{\beta \beta}\inv{\mat{A}_{\overline{\beta} \beta}} \vect{q}_{\overline{\beta}} \\
-\inv{\mat{A}_{\overline{\beta} \beta}}\vect{q}_{\overline{\beta}}
\end{bmatrix} \nonumber
\end{align}
If we set the values of the basic variables to zero, then solving for the values of the non-basic variables $\vect{z}_{\overline{\mathcal{B}}}$ and $\vect{w}_{\overline{\mathcal{B}}}$ entails only block inversion of $\mat{A}$.

The unmodified PPM~I operates in the following manner: \1 Find an index $i$ of a basic variable $x_i$ (where $x_i$ is either $w_i$ or $z_i$, depending which of the two is basic) such that $x_i < 0$; \2 swap the variables between basic and non-basic sets for index $i$ (\emph{e.g.}, if $w_i$ is basic and $z_i$ is non-basic, make $w_i$ non-basic and $z_i$ basic); \3 determine new values of $\vect{z}$ and $\vect{w}$; \4 repeat \1\hspace{.75mm}--\hspace{.1mm}\3 until no basic variable has a negative value. 

PPM~I requires few modifications, which are provided in Algorithm~\ref{alg:PPM}. First, the full matrix $\mat{N} \cdot \inv{\mat{M}} \cdot \tr{\mat{N}}$ is never constructed (such construction would require $O(n^3)$ time). Instead, Line~\ref{line:PPM:A} of the algorithm constructs a maximum $m \times m$ system; thus, that operation requires only $O(m^3)$ operations. Similarly, Lines 13--14 also leverage Theorem~\ref{thm:maxcard} in order to compute $\vect{w}^\dag$ and $\vect{a}^\dag$ efficiently (though these operations do not affect the asymptotic time complexity).
Assuming that the number of iterations for a pivoting algorithm is $O(n)$ in the size of the input,\footnote{Regardless of the pathological problem devised by Klee and Minty\cite{Klee:1972}, experience with the Simplex Algorithm on thousands of \emph{practical} problems shows that it requires fewer than $3n$ iterations and the expected time complexity for the Simplex Algorithm is polynomial~\cite{Shamir:1987,Spielman:2004}. We are unaware of research that shows these results are also applicable to pivoting methods for LCPs, though Cottle~\emph{et al.} claim $O(n)$ expected iterations~\cite{Cottle:1992}.} and that each iteration requires at most two pivot operations (each rank-1 update operation to a matrix factorization will exhibit time complexity $O(m^2)$), the asymptotic complexity of the modified PPM~I algorithm is $O(m^3 + m^2n)$. The termination conditions for the algorithm are not affected by our modifications.

\begin{algorithm*}
\begin{algorithmic}[1]
\caption{\label{alg:PPM}$\{ \vect{z}, \vect{w}\} = $\textsc{ LCP}$(\vect{N}, \vect{M}, \vect{f}^*)$ Solves a frictionless contact model using a modification of the Principal Pivoting Method~I Algorithm.}
\State $n \leftarrow $ rows($\mat{N}$)
\State $\vect{q} \leftarrow \mat{N}\cdot\vect{f}^*$
\State $i \leftarrow \argmin_i q_i$
\Comment {Check for trivial solution}
\If{$q_i \geq 0$}
\State \textbf{return}\ $\{ \vect{0}, \vect{q} \}$
\EndIf
\State $\overline{\mathcal{B}} \leftarrow \{ i \}$
\Comment {Establish initial nonbasic indices}
\State $\mathcal{B} \leftarrow \{ 1, \ldots, i-1, i+1, \ldots, n\}$
\Comment {Establish initial basic indices}
\While{\emph{true}}
\State $\mat{A} \leftarrow \mat{N}_{\overline{\mathcal{B}}} \cdot \inv{\mat{M}} \cdot \tr{\mat{N}_{\overline{\mathcal{B}}}}$ \label{line:PPM:A}
\State $\vect{b} \leftarrow \mat{N}_{\overline{\mathcal{B}}} \cdot \vect{f}^*$
\State $\vect{z}^\dag \leftarrow \inv{\mat{A}} \cdot -\vect{b}$ \;\nbsp
\Comment {Solve for non-basic components of $\vect{z}$}
\State $\vect{a}^\dag \leftarrow \inv{\mat{M}} \cdot \tr{\mat{N}_{\overline{\mathcal{B}}}} \vect{z}^\dag + \vect{f}^*$
\State $\vect{w}^\dag \leftarrow \mat{N} \cdot \vect{a}^\dag$
\State $i \leftarrow \argmin_i w^\dag_i$
\Comment {Find the index for moving into the non-basic set (if any)}
\If{$w^\dag_i \geq 0$} 
  \State $j \leftarrow \argmin_i \vect{z}^\dag_i$
  \Comment {No index to move into the non-basic set; look whether there is an index to\\  \nbsp\qquad\qquad\qquad\qquad\qquad\qquad\qquad\qquad\qquad\; move into the basic set}
  \If{$z^\dag_j < 0$}
    \State $k \leftarrow \overline{\mathcal{B}}(j)$
    \State $\mathcal{B} \leftarrow \mathcal{B} \cup \{ k \}$
    \Comment {Move index $k$ into the basic set}
    \State $\overline{\mathcal{B}} \leftarrow \overline{\mathcal{B}} - \{ k \}$
    \State \textbf{continue}
  \Else
    \State $\vect{z} \leftarrow \vect{0}$
    \State $\vect{z}_{\overline{\mathcal{B}}} \leftarrow \vect{z}^\dag$
    \State $\vect{w} \leftarrow \vect{0}$
    \State $\vect{w}_{\mathcal{B}} \leftarrow \vect{w}^\dag$
    \State \textbf{return} $\{ \vect{z}, \vect{w} \}$
  \EndIf
\Else
  \State $\overline{\mathcal{B}} \leftarrow \overline{\mathcal{B}} \cup \{ i \}$
  \Comment{Move index $i$ into the non-basic set}
  \State $\mathcal{B} \leftarrow \mathcal{B} - \{ i \}$
  \State $j \leftarrow \argmin_i \vect{z}^\dag_i$
  \Comment {Look whether there is an index to move into the basic set}
  \If{$z^\dag_j < 0$}
    \State $k \leftarrow \overline{\mathcal{B}}(j)$
    \State $\mathcal{B} \leftarrow \mathcal{B} \cup \{ k \}$
    \Comment {Move index $k$ into the basic set}
    \State $\overline{\mathcal{B}} \leftarrow \overline{\mathcal{B}} - \{ k \}$
  \EndIf
\EndIf
\EndWhile
\end{algorithmic}
\end{algorithm*}

\section{No-slip contact model}
\label{section:no-slip}
A contact model without slip requires a velocity-level contact model in accordance with Lynch and Mason's finding that sliding can occur with infinite friction at the acceleration level~\cite{Lynch:1995}. The no-slip friction contact model uses the first-order approximation and builds on Equations~\ref{eqn:vel-system-dynamics}, \ref{eqn:vel-bilateral-constraint}, and \ref{eqn:vel-normal-comp} by dictating that the tangential velocity at each contact must be zero at $\vect{v}(t+\Delta t)$:
\begin{align}
\mat{S}\vect{v}(t+\Delta t) & = \vect{0} \label{eqn:Sv0} \\
\mat{T}\vect{v}(t+\Delta t) & = \vect{0} \label{eqn:Tv0}
\end{align}

Forming these five equations into a MLCP and using the variable definitions in Section~\ref{section:LCPs} yields:
\begin{align}
\mat{A} & \equiv \begin{bmatrix}
\mat{M} & -\tr{\mat{J}} & -\tr{\mat{S}} & -\tr{\mat{T}}  \\
\mat{J}  & \mat{0} & \mat{0} & \mat{0} \\
\mat{S} & \mat{0} & \mat{0} & \mat{0}  \\
\mat{T} & \mat{0} & \mat{0} & \mat{0}
\end{bmatrix}\\
\mat{C} & \equiv \begin{bmatrix}
-\tr{\mat{N}} \\
\mat{0} \\
\mat{0} \\
\mat{0} \\
\end{bmatrix}\\
\mat{D} & \equiv -\tr{\mat{C}} \\
\mat{B} & \equiv \mat{0}\\
\end{align}
\begin{align}
\vect{x} & \equiv \begin{bmatrix} \vect{v}(t+\Delta t) \\ \vect{f}_j \\ \vect{0} \\ \vect{0} \end{bmatrix}\\
\vect{y} & \equiv \vect{f}_n\\
\vect{g} & \equiv \begin{bmatrix} -\mat{M}\vect{v}(t) \\ \vect{0} \\ \vect{0} \\ \vect{0} \end{bmatrix}\\
\vect{h} & \equiv \vect{0}
\end{align} 
The matrix $\mat{A}$ may be singular, which would prevent us from converting the MLCP to a standard LCP. However, if $\mat{J}$, $\mat{S}$, and $\mat{T}$ have full row rank \emph{or we identify the largest row blocks of those matrices such that full row rank is attained}, $\mat{A}$ is invertible without affecting the solution to the MLCP. Algorithm~\ref{alg:find-indices} performs exactly this task.

\begin{algorithm}
\caption{\textsc{Find-Indices}($\mat{M}, \mat{J}, \mat{S}, \mat{T})$, determines the row indices ($\mathcal{J}$, $\mathcal{S}$, and $\mathcal{T}$) of $\mat{J}$, $\mat{S}$, and $\mat{T}$ such that the equality matrix $\mat{A}$ (Equation~\ref{eqn:MLCP-begin}) is non-singular. \label{alg:find-indices}}
\begin{algorithmic}[1]
\State $\mathcal{J} \leftarrow \emptyset$
\State $\mathcal{S} \leftarrow \emptyset$
\State $\mathcal{T} \leftarrow \emptyset$
\For {$i = 1, \ldots r$}
\Comment $r$ is the number of bilateral constraint equations
\State $\mathcal{J}^* \leftarrow \mathcal{J} \cup \{ i \}$
\State Set $\mat{X} \leftarrow \tr{\mat{J}}_{\mathcal{J}^*}$
\If {$\tr{\mat{X}}\inv{\mat{M}}\mat{X}$ not singular} \label{line:XTiMX:1}
\State $\mathcal{J} \leftarrow \mathcal{J}^*$
\EndIf
\EndFor
\For {$i = 1, \ldots, n$}
\Comment $n$ is the number of contacts
\State $\mathcal{S}^* \leftarrow \mathcal{S} \cup \{ i \}$
\State Set $\mat{X} \leftarrow \begin{bmatrix} \tr{\mat{J}}_\mathcal{J} & \tr{\mat{S}}_{\mathcal{S}^*} & \tr{\mat{T}}_\mathcal{T}\end{bmatrix}$
\If {$\tr{\mat{X}}\inv{\mat{M}}\mat{X}$ not singular} \label{line:XTiMX:2}
\State $\mathcal{S} \leftarrow \mathcal{S}^*$
\EndIf
\State $\mathcal{T}^* \leftarrow \mathcal{T} \cup \{ i \}$
\State Set $\mat{X} \leftarrow \begin{bmatrix} \tr{\mat{J}}_\mathcal{J} & \tr{\mat{S}}_\mathcal{S} & \tr{\mat{T}}_{\mathcal{T}^*} \end{bmatrix}$
\If {$\tr{\mat{X}}\inv{\mat{M}}\mat{X}$ not singular} \label{line:XTiMX:3}
\State $\mathcal{T} \leftarrow \mathcal{T}^*$
\EndIf
\EndFor
\State \Return $\{ \mathcal{J}, \mathcal{S}, \mathcal{T} \}$
\end{algorithmic}
\end{algorithm}

The singularity check on Lines~\ref{line:XTiMX:1},~\ref{line:XTiMX:2}, and~\ref{line:XTiMX:3} of Algorithm~\ref{alg:find-indices} is best performed using a Cholesky factorization; if the factorization is successful, the matrix is non-singular. Given that $\mat{M}$ is non-singular (it is symmetric and positive definite), the maximum size of $\mat{X}$ in Algorithm~\ref{alg:find-indices} is $m \times m$; if $\mat{X}$ were larger, it would be singular (see Lemma~\ref{lemma:rank}). 

Given this information, the time complexity of Algorithm~\ref{alg:find-indices} is dominated by Lines~\ref{line:XTiMX:1},~\ref{line:XTiMX:2}, and~\ref{line:XTiMX:3}. As $\mat{X}$ changes by at most one row and one column per Cholesky factorization, singularity can be checked by  $O(m^2)$ updates to an initial $O(m^3)$ Cholesky factorization. The overall time complexity is \mbox{$O(m^3 + nm^2)$}.

\subsection{Resulting systems}
Using Equations~\ref{eqn:MLCP-LCP1}~and~\ref{eqn:MLCP-LCP2}), the LCP matrix $\mat{F}$ and vector $\vect{e}$ are equivalent to:
\begin{align}
\mat{F} & \equiv \mat{N}\inv{\mat{X}}\tr{\mat{N}} \\
\vect{e} & \equiv \mat{N}\inv{\mat{X}}\mat{M}\vect{v}(t)
\end{align}

As in Section~\ref{section:viscous-model}, 
$\mat{F}$ must be symmetric and positive-semi-definite and---as noted in Section~\ref{section:viscous-model}---Baraff's algorithm~\cite{Baraff:1994} guarantees that a solution to this LCP exists.

The $\mat{S}\vect{v}(t+\Delta t) = \mat{T}\vect{v}(t+\Delta t) = \vect{0}$ constraints (Equations~\ref{eqn:Sv0}~and~\ref{eqn:Tv0}) and solvability of the LCP contrast with the finding of Lynch and Mason~\cite{Lynch:1995}, who showed that sliding with infinite friction is possible. The admittance of impulsive forces has resolved this ``paradox'' analogously to the manner in which contact models like~\cite{Stewart:1996,Anitescu:1997} resolved Painlev\'{e}'s Paradox~\cite{Painleve:1895} and other inconsistent contact configurations~\cite{Stewart:2000a}.

\section{Experiments}
We tested the contact models using two common contact scenarios in robotics, grasping and locomotion, in order to assess speed and numerical stability. These experiments can be reproduced using the experimental setup described at \urlx{https://github.com/PositronicsLab/no-slip-and-viscous-experiments}. 

\subsection{Grasping experiment}
\label{section:exp:grasping}
We used \RPISIM (\urlx{https://code.google.com/p/rpi-matlab-simulator}) to simulate a force-closure grasping scenario (depicted in Figure~\ref{fig:grasp}) on a Macbook Air with 1.8 GHz Intel Core i5 CPU.  Twelve contact points were generated between each pair of boxes\footnote{The equal size boxes contacting in the manner in Figure~\ref{fig:grasp} yields degenerate contact normals at the box corners; the RPI simulator treats this problem by duplicating each contact with all three possible directions for the contact normal.}, yielding 36 contact points total. For the contact model with Coulomb friction (the Stewart-Trinkle model~\cite{Stewart:2000}), a friction ``pyramid'' (four sided approximation to the friction cone) was used, yielding six LCP variables per contact (\emph{i.e.}, 216 variables total). The RPI simulator allowed us to substitute the Stewart-Trinkle model with the no-slip friction model readily, which resulted in only 36 LCP variables. The simulation was run using a step size of $0.01$ for ten iterations ($\frac{1}{10}$ of one second of simulated time); the grasped objects would tend to fall from the gripper after ten iterations \emph{only when using Stewart-Trinkle} (due to numerical issues with Lemke's Algorithm, to be discussed below). Lemke's Algorithm was implemented using \LEMKE~\cite{LEMKE}.

The $\mat{P}_0$ matrix resulting from the no-slip friction model allowed us to employ the modified PPM solver and \texttt{MATLAB}'s \texttt{quadprog} solver (with the active-set algorithm) to solve the LCP. We used Lemke's Algorithm~\cite{Lemke:1965}, employing Tikhonov regularization~\cite{Cottle:1992} as necessary, to solve the Stewart-Trinkle model.  No low-rank updates were used in our implementation of Algorithm~\ref{alg:find-indices}.

\begin{table}[htdp]
\begin{center}
\begin{tabular}{|l|r|}
\hline
Contact model & Running time (mean $\pm\ \sigma$) \\ \hline
Stewart-Trinkle (Lemke's Algorithm) & 10.9681s $\pm$ 2.1812s  \\
\quad $\mu_c =100.0, \mu_v=0.0$ &\\
No-slip (active-set QP solver) & 1.9892s $\pm$ 0.2640s \\
No-slip (modified PPM) & 1.6680s $\pm$ 0.3669s \\ \hline
\end{tabular}
\end{center}
\label{table:experiment}
\caption{Mean running times for the grasping experiment. Ten trials were run for each method. Timings include all aspects of the simulation (including collision detection).}
\end{table}%

\begin{figure}[htbp]
\begin{center}
\includegraphics[width=.8\linewidth]{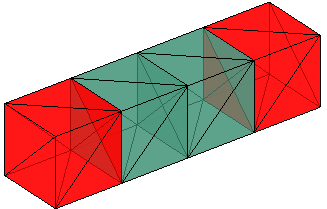}
\caption{A depiction of the grasping experiment described in Section~\ref{section:exp:grasping}. The two red boxes act as grippers and push inward. Gravity pushes downward. Given sufficient friction ($\mu = \infty$), the grippers should ideally keep the blue boxes grasped using force closure.}
\label{fig:grasp}
\end{center}
\end{figure}

This experiment yielded several findings. As expected, reducing the LCP variables by a factor of five (216 variables to 36 variables) results in much faster solutions ($451$--$558\%$ faster mean, depending on the solver). Fewer variables also results in less rounding error; the no-slip approach was able to model the grasping scenario reliably for at least 100 iterations (again, compared to around ten iterations for Stewart-Trinkle).  Of the twelve contacts per pair of boxes, it was only necessary to apply forces to two contacts, which the modified PPM method was able to exploit: it ran nearly 20\% faster than the \texttt{quadprog} algorithm (mean and maximum numbers of pivot operations were observed to be 5.5 and 7, respectively). This performance differential is considerable given that our modified PPM algorithm was not implemented as a MEX file and that our implementation does not use low-rank updates to maximize performance).   

\subsection{Locomotion experiments}
\label{section:exp:locomotion}
We used the \Moby simulator (\urlx{https://github.com/PositronicsLab/Moby}) to simulate a quadrupedal robot walking on a terrain map (see Figure~\ref{fig:locomotion}) over ten seconds. An event-driven method (see~\cite{Brogliato:2002} for a description of this paradigm) is used to simulate the system instead of the time-stepping approaches used in \ODE, \Bullet, and \RPISIM; popular implementations of this approach are susceptible to energy gain when correcting interpenetration~\cite{Stewart:2000}, which destabilizes our robot in the process. 

The integration method used for the no-slip model experiment is symplectic Euler (St\"ormer-Verlet) with a step size of 0.001, while fourth-order Runge-Kutta integration was used for the viscous model experiment with identical step size. Our approach using the former integrator allows only the no-slip model to be activated, while our approach using the latter integrator permits the acceleration-level viscous model to be used for sustained contacts. When impacts occur (\emph{e.g.}, on initial foot/ground contact) in the latter approach, the simulation uses an inelastic impact model~\cite{Drumwright:2010} with purely viscous friction. 

Locomotion experiments were run on an Intel Xeon 2.27GHz desktop computer. All aspects of the simulation, including forward dynamics computations, collision detection, and controls (which performs dynamics calculations) were accounted for in results; single LCP solves generally run too quickly to obtain timings for just that operation. We point out that the simulations run considerably slower than similar systems modeled using, \emph{e.g.}, \ODE, but the goals of the two simulators. \ODE uses approximate solves and permits interpenetration. \Moby aims to provide a \emph{verifiable} simulator, \emph{i.e.}, one that adheres closely to the rigid body dynamics models (which means that interpenetration is impermissible). 

Results for the no-slip experiment are provided in Table~\ref{table:locomotion-noslip}, which shows a speedup of nearly 28\%. The minimum and maximum number of contact points generated in the experiment is 1 and 30, respectively; the mean number of contact points is 6, and the standard deviation is 8. Thus, simulations with greater numbers of contacts could expect greater performance differentials.  

Results for the viscous experiment are provided in Table~\ref{table:locomotion-viscous}, which shows a speedup of over 37\%. We note that the viscous friction experiments required
significantly longer to run than the no-slip experiments. We hypothesize that this disparity is due to the behavior of the simulation when applying the viscous model, which tends to produce rapid movements upon contact. Those rapid movements, which appear due to some sensitivity  in the underlying ordinary differential equations, slow the simulator's continuous collision detection system (see~\cite{Mirtich:1996vt} for a description of that system). 

\begin{figure}[htbp]
\centering
\includegraphics[width=\linewidth]{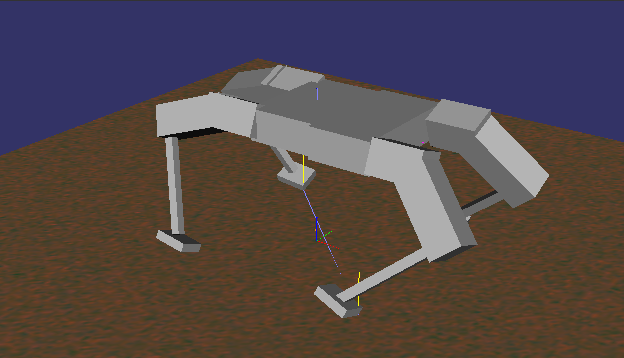}
\caption{A depiction of the quadrupedal robot walking on a terrain map in the locomotion experiment, as described in Section~\ref{section:exp:locomotion}. The no-slip and viscous friction models were both assessed.}
\label{fig:locomotion}
\end{figure}

\begin{table}[htdp]
\centering
\begin{tabular}{|l|r|}
\hline
Contact model & Running time \\ \hline
Drumwright-Shell & 416.284s \\
\quad $\mu_c=100.0, \mu_v = 0.0$ & \\
No-slip (modified PPM) & 301.796s \\
\hline
\end{tabular}
\label{table:locomotion-noslip}
\caption{Times required to simulate the quadruped locomotion scenario described in Section~\ref{section:exp:locomotion} under a no-slip contact model. Timings include all aspects of the simulation (including collision detection).}
\end{table}%

\begin{table}[htdp]
\centering
\begin{tabular}{|l|r|}
\hline
Contact model & Running time \\ \hline
Viscous (Lemke's Algorithm) & 2936.36s \\
\quad $\mu_c=0.0, \mu_v = 0.1$ &  \\
Viscous (modified PPM) & 2139.73s \\
\hline
\end{tabular}
\label{table:locomotion-viscous}
\caption{Times required to simulate the quadruped locomotion scenario described in Section~\ref{section:exp:locomotion} under a purely viscous friction model. Timings include all aspects of the simulation (including collision detection).}
\end{table}%

\section{Conclusion}
We presented an algorithm for rapidly computing two rigid contact models without Coulomb friction that have proven useful in certain modeling and simulation applications for robotics. We showed how these models exhibit both asymptotic computational complexity and significant running time advantages over rigid models with Coulomb friction. While we do not expect these special-case models to replace rigid models with Coulomb friction, the former serve as computationally efficient alternatives as applications allow.

\section{Acknowledgements}
We thank Sam Zapolsky for providing the quadruped model and the locomotion controller. This work was funded by NSF CMMI-110532.

\bibliographystyle{abbrv}
\bibliography{paper}

\end{document}